\documentclass[twoside, 11pt]{article}

\usepackage{graphicx, subcaption}
\usepackage[margin=1in, footnotesep=0.5in]{geometry}

\usepackage[sort&compress, numbers]{natbib} 

\usepackage{amsfonts, amsmath, amssymb, amsthm}
\usepackage[shortlabels]{enumitem}
\usepackage{booktabs}
\usepackage{mathtools}
\mathtoolsset{showonlyrefs}

\usepackage{color}
\definecolor{orange}{RGB}{195,125,0}
\definecolor{customdarkred}{RGB}{150,0,0}
\definecolor{customdarkgreen}{RGB}{0,150,0}
\definecolor{customdarkblue}{RGB}{0,0,150}

\usepackage{hyperref}
\hypersetup{colorlinks=true, linkcolor=customdarkred, citecolor=customdarkgreen, urlcolor=customdarkblue}
\usepackage{url}
\usepackage[linesnumbered, ruled]{algorithm2e}
\usepackage{authblk}

\def\level{L}
\def\up{U}
\def\supp{\text{supp}}

\newtheorem{theorem}{Theorem}[section]
\newtheorem{proposition}[theorem]{Proposition}
\newtheorem{lemma}[theorem]{Lemma}

\theoremstyle{definition}
\newtheorem{definition}[theorem]{Definition}
\newtheorem{example}[theorem]{Example}

\theoremstyle{remark}
\newtheorem{remark}[theorem]{Remark}

\numberwithin{equation}{section}
\numberwithin{figure}{section}

\def\beq{\begin{equation}} % \setcounter{equation}{1}}
\def\eeq{\end{equation}}
\def\beqn{\begin{eqnarray*}}
\def\eeqn{\end{eqnarray*}}
\def\Bitem{\begin{itemize}\setlength{\itemsep}{.2in}}
\def\bitem{\begin{itemize}\setlength{\itemsep}{.05in}}
\def\eitem{\end{itemize}}
\def\Benum{\begin{enumerate}\setlength{\itemsep}{.2in}}
\def\benum{\begin{enumerate}\setlength{\itemsep}{.05in}}
\def\eenum{\end{enumerate}}
\def\bmult{\begin{multline*}}
\def\emult{\end{multline*}}
\def\bcenter{\begin{center}}
\def\ecenter{\end{center}}
\def\bframe{\begin{frame}}
\def\eframe{\end{frame}}

\newcommand{\thmref}[1]{Theorem~\ref{thm:#1}}
\newcommand{\prpref}[1]{Proposition~\ref{prp:#1}}

\newcommand{\lemref}[1]{Lemma~\ref{lem:#1}}
\newcommand{\secref}[1]{Section~\ref{sec:#1}}
\newcommand{\figref}[1]{Figure~\ref{fig:#1}}

\newcommand{\defref}[1]{Definition~\ref{def:#1}}

\DeclareMathOperator{\diam}{diam}

\def\cA{\mathcal{A}}

\def\cC{\mathcal{C}}

\def\cF{\mathcal{F}}

\def\cH{\mathcal{H}}

\def\cL{\mathcal{L}}

\def\cN{\mathcal{N}}

\def\cS{\mathcal{S}}

\def\cX{\mathcal{X}}

\def\bbI{\mathbb{I}}

\def\bbR{\mathbb{R}}

\def\bbZ{\mathbb{Z}}

\def\eps{\epsilon}
\def\implies{\ \Rightarrow \ }
\def\iff{\ \Leftrightarrow \ }
\DeclareMathOperator{\sign}{sign}

\def\1{\mathbbm{1}}

\title{An Axiomatic Definition of Hierarchical Clustering}
\author[1,2]{Ery Arias-Castro}
\author[1]{Elizabeth Coda} 
\affil[1]{\small Department of Mathematics, University of California, San Diego} 
\affil[2]{\small Halıcıoğlu Data Science Institute, University of California, San Diego}
\date{}

\begin{document}
\maketitle
\thispagestyle{empty}

\begin{abstract}
In this paper, we take an axiomatic approach to defining a population hierarchical clustering for piecewise constant densities, and in a similar manner to Lebesgue integration, extend this definition to more general densities. When the density satisfies some mild conditions, e.g., when it has connected support, is continuous, and vanishes only at infinity, or when the connected components of the density satisfy these conditions, our axiomatic definition results in Hartigan's definition of cluster tree.  
\end{abstract}

\section{Introduction}
\label{sec:intro}

Clustering, informally understood as the grouping of data, is a central task in statistics and computer science with broad applications. Modern clustering algorithms originated in the work of numerical taxonomists, who developed methods to identify hierarchical structures in the classification of plant and animal species. Since then clustering has been used in disciplines such as medicine, astronomy, anthropology, economics, etc., with aims such as exploratory analysis, data summarization, the identification of salient structures in data, and information organization. 

The notion of a ``good'' or ``accurate'' clustering varies between fields and applications. For example, to some computer scientists, the correct clustering of a dataset is often defined as the solution to an optimization problem (think K-means) and a good algorithm either solves or approximates a solution to this problem, ideally with some guarantees~\cite{puzicha2000, dasgupta2016}. From this perspective, the dataset is viewed as fixed, and the cluster definition is based on the data alone~\cite{hennig2015}. Moreover, depending on the particular application, how good a clustering is deemed to be may be further loosened, such as in the task of image segmentation, where a good clustering need only ``extract the global impression of an image'' according to~\cite{shi2000}. 

Even as this view of clustering is widespread well outside computer science, it is not satisfactory from a statistical inference perspective. Indeed, in statistics, it is typically assumed that the sample is representative of an underlying population and a clustering method, to be useful, should inform the analyst about that population. This viewpoint calls for a definition of clustering at the population level. When the sample is assumed iid from an underlying distribution representing the population, by clustering we mean a partition of the support of that distribution, and in that case, a clustering of the sample is deemed ``good'' or ``accurate'' by reference to the population clustering --- and a clustering algorithm is a good one if it is consistent, meaning, exact, in the large-sample limit. This reference to the population is what gives meaning to statistical inference, and to questions such as whether an observed cluster is ``real'' or not. 

However, there is not a generally accepted definition of clustering at the population level. One popular approach assumes that the data is drawn from a mixture model $f = \sum_{i=1}^{k} \alpha_i f_i$  and the population level clustering consists of $k$ clusters corresponding to the mixture components \cite{fraley2002, bouveyron2019}. If the underlying density is not a mixture, it can be approximated by a mixture model (typically chosen to be a multivariate Gaussian), though this requires a modeling choice. This approximation may require a very large number of components to approximate well, resulting in an artificially large number of clusters. Moreover, even if the density is a mixture, under this definition, a unimodal mixture could have multiple clusters. Alternatively, in the gradient flow approach to defining the population level clustering, often attributed to \citet{fukunaga1975}, each point is assigned to the nearest mode (i.e., local maximum) in the direction of the gradient. Thus, at least when the density has Morse regularity, the clusters correspond to the basin of attraction of each mode. Though this definition relies on assumptions about the smoothness of the density and does not account for arbitrarily flat densities~\cite{menardi2016}, it overcomes some of the  described difficulties of the mixture model clustering. If the components in the mixture model are well-separated, this definition results in a similar clustering to the mixture-based definition~\cite{chacon2020}. 

Taking a hierarchical perspective of clustering, \citet{hartigan1975} has proposed a population-level cluster tree, where clusters correspond to the maximally connected components of density upper level sets. Though Hartigan provides minimal motivation for this definition beyond observing that each cluster $C$ in his tree ``conforms to the informal requirement that $C$ is a high-density region surrounded by a low-density region'' \cite{hartigan1975}, this is generally accepted as the definition of hierarchical clustering at the population level and has been used in subsequent works~\cite{eldridge2015, wang2019, chaudhuri2014, kim2016,balakrishnan2013,steinwart2011}. It has been shown that Hartigan's definition of hierarchical clustering is fully compatible with Fukunaga and Hostetler's definition of clustering~\cite{arias2023}. 

Several works have explored axiomatic approaches to defining clustering algorithms that take as input a finite number of data points~\cite{kleinberg2002, zadeh2009, ben2008, puzicha2000, jardine1968, carlsson2010}. In each of these, the authors state desirable characteristics of a clustering function (or clustering criterion) and identify algorithms that satisfy these requirements, or in  case of~\cite{kleinberg2002}, prove the non-existence of such an algorithm. Inspired by these axiomatic approaches, we devote a significant portion of the paper to developing an axiom-based definition of a population hierarchical clustering. We focus on hierarchical clustering, rather than flat clustering, as we find this to be a simpler starting point. We recover Hartigan's definition of hierarchical clustering for densities with connected support that satisfy continuity and some additional mild assumptions, as well as densities with finitely many connected components, each of which satisfy these conditions. 

\subsection{Related Work}
\label{sec:related_work} While we take an axiomatic approach to defining the population-level hierarchical clustering, several previous works have explored axiomatic approaches to defining clustering algorithms. The most famous of which might be that of~\cite{kleinberg2002}, where three axioms are proposed (scale-invariance, richness, and consistency) and an impossibility theorem is established, proving that no clustering algorithm can simultaneously satisfy all three axioms. (The `consistency' axiom is not in the statistical sense, but refers to the property that if within-cluster distances are decreased and between-cluster distances are enlarged, then the output clustering does not change.)

However, as has been pointed out by others~\cite{ben2008, strazzeri2022, Cohen-Addad2018}, including Kleinberg himself in Section 5 of the same article~\cite{kleinberg2002}, the consistency property may not be so desirable. Rather, a relaxation of this property, in which a refinement or coarsening of the clusters is allowed, may be more appropriate. Kleinberg states that clustering algorithms that satisfy scale-invariance, richness, and this relaxed notion of refinement-coarsening consistency do exist and clustering algorithms that satisfy scale-invariance, near richness, and refinement consistency also exist. 
This was, in a sense, confirmed by \citet{Cohen-Addad2018}, who allow the number of clusters to vary with the refinement.
\citet{zadeh2009} show that, if the number of clusters that a clustering algorithm can return is fixed at $k$, there exist clustering algorithms that satisfy scale-invariance, $k$-richness, and consistency (in the original sense). They also show that single linkage is the unique clustering algorithm returning a fixed number of clusters simultaneously satisfying these axioms and two additional axioms. 

\citet{puzicha2000} consider clustering data via optimization of a suitable objective function and define a suitable objective function with a set of axioms. Though their axioms are somewhat strong, requiring the objective function have an additive structure, they show that only one of the objective functions considered satisfies all of their axioms. \citet{ben2008} also propose a set of axioms which strongly parallel Kleinberg's axioms for a clustering quality measure function and show the existence of functions satisfying these axioms. 

In the 1960s and 1970s, there were a number of articles examining the axiomatic foundation of clustering. \citet{cormack1971review} provides a comprehensive review.
For example, \citet{jardine1968, jardine1967structure, SIBSON1970405}, list axioms that, according to them, a hierarchical clustering algorithm should satisfy, and then state that single linkage is the only algorithm they are aware of that satisfies all of their axioms. 
More recently, \citet{carlsson2010} propose their own sets of axioms for hierarchical clustering, and then prove that single linkage is the only algorithm that satisfies them. Though this result has been presented as a demonstration that Kleinberg's impossibility theorem does not hold when hierarchical clustering algorithms are considered, this connection is somewhat unclear to us, as the proposed axioms do not mirror Kleinberg's axioms very precisely.  

\subsection{Content}
The organization of the paper is as follows. 
\secref{prelim} provides some basic notation and definitions.
In \secref{axioms}, we take an axiomatic approach to defining a hierarchical clustering for a piecewise constant density with connected support. 
In \secref{continuous}, we extend this definition to continuous densities, first to densities with connected support, and then to more general densities. 
\secref{discussion} is a discussion section where we go over some extensions, some practical considerations, and also discuss some outlook on flat clustering.
In an appendix, we provide a close examination of the merge distortion metric in \secref{merge}, and provide further technical details for the special case of a Euclidean space in \secref{euclidean}.

\section{Preliminaries}
\label{sec:prelim}
Throughout this paper, we will work with a metric space $(\Omega, d)$. For technical reasons, we assume it is locally connected, which is for example the case if the balls are connected. This is so that the connected components of an open set are connected.\footnote{This is, in fact, an equivalence, the proof of which is left as an exercise in Armstrong's textbook \cite[Ch 3]{alma991008609639706535}.}

In principle, we would equip  this metric space with a suitable Borel measure, and consider densities with respect to that measure. As it turns out, this equipment is not needed as we can directly work with non-negative functions. We will do so for the most part, although we will sometimes talk about densities.

For a set $A \subseteq \Omega$, we let $\text{int}(A)$ or $A^\circ$ denote its interior and $\text{clo}(A)$ or $\overline{A}$ denote its closure; we also let $\text{cc}(A)$ denote the collection of its connected components.
For a function $f :\Omega \to \bbR$, its support is $\supp(f) = \text{clo}\{f \ne 0\}$, and for $\lambda \in \bbR$, its upper $\lambda$-level set is defined as $\{f \ge \lambda\}$, denoted $\up_\lambda$ when there is no ambiguity. 

\begin{definition}[\textbf{Hierarchical clustering or cluster tree}]
\label{def:cluster tree}
A hierarchical clustering, or cluster tree, of $\cX \subseteq \Omega$ is a collection of connected subsets of $\cX$, referred to as clusters, that has a nested structure in that two clusters are either disjoint or nested.
\end{definition}

A hierarchical clustering of a function $f$ is understood as a hierarchical clustering of its support $\supp(f)$. 
Hartigan's definition of hierarchical clustering for a density is arguably the most well-known one. 

\begin{definition}[\textbf{Hartigan cluster tree}]
\label{def:hartigan tree}
The Hartigan cluster tree of a function $f$, which will be denoted $\cH_f$,  is the collection consisting of the maximally connected components of the upper $\lambda$-level sets of $f$ for all $\lambda > 0$. $\cH_f$ is a hierarchical clustering of $\supp(f)$.
\end{definition}

A dendrogram is commonly understood as the output of a hierarchical clustering algorithms such as single-linkage clustering. It turns out to be simpler to work with dendrograms instead of directly with cluster trees \cite{carlsson2010, eldridge2015}. 

\begin{definition}[\textbf{Dendrogram}]
A dendrogram is a cluster tree equipped with a real-valued non-increasing function defined on the cluster tree called the height function. A dendrogram is thus of the form $(\cC, h)$ where $\cC$ is a cluster tree and $h : \cC \to \bbR$ is such that $h(C') \ge h(C)$ whenever $C' \subseteq C$.
\end{definition}

The Hartigan tree of a function $f$ is naturally equipped with the following height function 
\begin{equation}
\label{height}
h_f(C) = \inf\limits_{x \in C} f(x).
\end{equation} 
Note that this function has the required monotonicity.

\citet{eldridge2015} introduced the merge distortion metric to compare dendrograms. It is based on the notion of merge height, which gives the height at which two points stop belonging to the same cluster, or equivalently, the height of the smallest cluster that contains both points.

\begin{definition}[\textbf{Merge height}]
Let $(\cC, h)$ be a dendrogram. The merge height of two points $x,y \in \Omega$ is defined as
\begin{equation}
\label{merge height}
m_{(\cC, h)}(x,y) = \sup\limits_{\substack{C \in \cC \\  x,y \in C}} h(C).
\end{equation}
\end{definition}

For the special case of an Hartigan cluster tree,
\begin{equation}
\label{merge height hartigan}
m_f(x,y) = m_{(\cH_f, h_f)}(x,y) = \sup\limits_{\substack{C \text{ connected } \\ x,y \in C}} \inf_{z \in C} f(z).
\end{equation}

\begin{definition}[\textbf{Merge distortion metric}]
\label{def:merge metric}
Let $(\cC, h)$ and $(\cC', h')$ be two dendrograms. Their merge distortion distance is defined as
$$
d_M((\cC,h), (\cC',h')) = \sup\limits_{x,y \in \Omega} \big|m_{(\cC, h)}(x,y) - m_{(\cC', h')}(x,y)\big|. 
$$
\end{definition}

\begin{figure}[h!]
    \centering
    \includegraphics[width=0.8\linewidth]{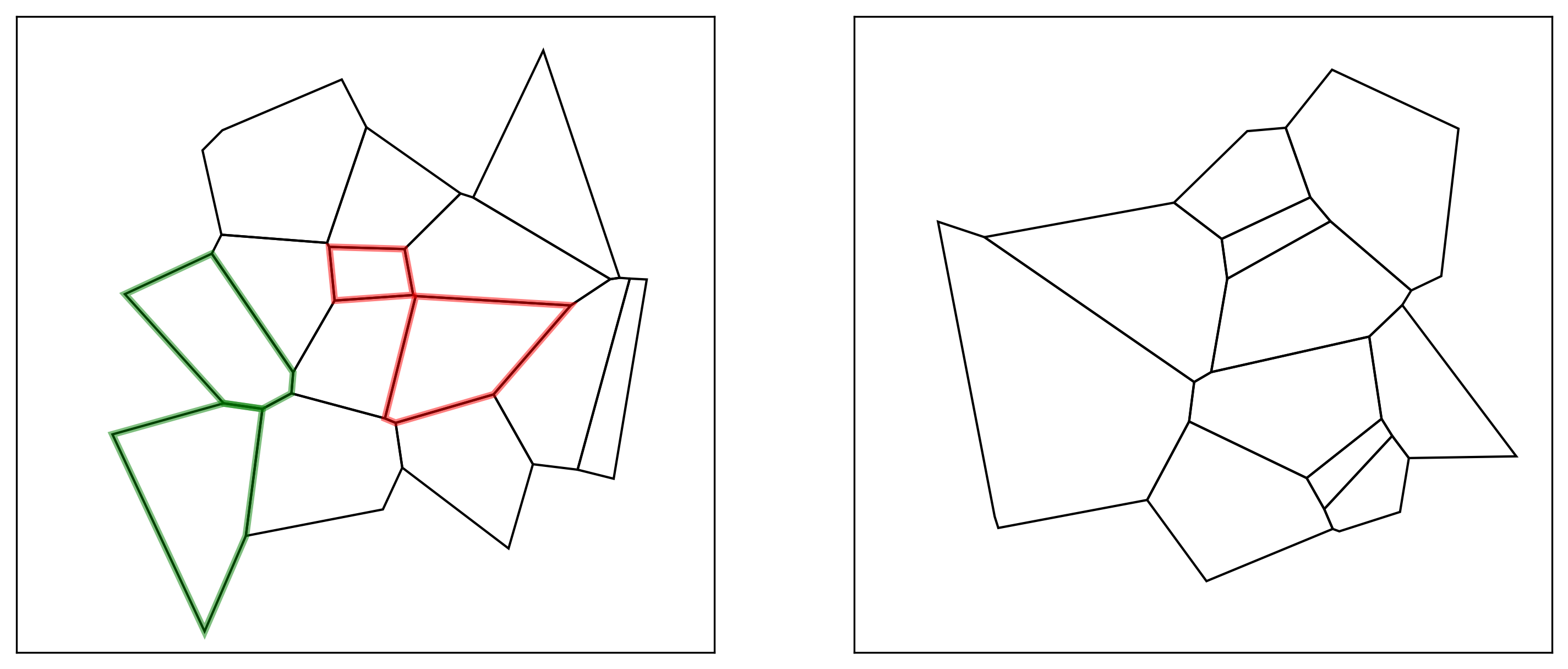}
    \caption{Left: A collection of sets with neighboring regions (green) and non-neighboring regions (red). This collection of sets does not have the internally connected property. Right: A collection of sets with the internally connected property.}
    \label{fig:neighbors}
\end{figure}

The merge distortion metric has the following useful property \cite[Th 17]{eldridge2015}.

\begin{lemma}
\label{lem:th17}
For two functions $f$ and $g$, 
\begin{equation}
\label{th17}
d_M((\cH_f, h_f), (\cH_g, h_g)) \le \|f - g\|_\infty.
\end{equation}
\end{lemma}
\begin{proof}
The arguments in \cite{eldridge2015} are a little unclear (likely due to typos), but correct arguments are given in \cite[Lem 1]{kim2016}. We nonetheless provide a concise proof as it is instructive.
%We use abbreviated notation. 
Take $x,y \in \Omega$, and let $s = m_f(x,y)$ and $t = m_g(x,y)$. We need to show that $|s-t| \le \eta := \|f-g\|_\infty$. 
For any $\eps > 0$, by \eqref{merge height hartigan}, there is a connected set $C$ containing $x$ and $y$ such that $f(z) \ge s - \eps$ for all $z \in C$. Since this implies that $g(z) \ge s - \eps - \eta$ for all $z \in C$, by \eqref{merge height hartigan} again, this yields $t \ge s - \eps - \eta$. We have thus shown that $s \le t + \eta + \eps$, and can show that $t \le s + \eta + \eps$ in exactly the same way, which combined allows us to obtain that $|s-t| \le \eta + \eps$. With $\eps >0$ arbitrary, we conclude. 
\end{proof}

The merge distortion metric has gained some popularity in subsequent works that discuss the consistency of hierarchical methods~\cite{kim2016,wang2019}. In \secref{merge} we discuss some limitations and issues with the merge distortion metric, which is in fact a pseudometric on general cluster trees. However, in the context in which we use the metric, these issues are not significant. 

We also introduce the notion of neighboring sets. 
Throughout, we adopt the convention that the empty set is disconnected.

\begin{definition}[\textbf{Neighboring regions}]
\label{def:neighbors}
Given a collection of sets $\cA = \{A_i \}$, we define the neighborhood of $A_i$ as
\begin{equation}
\label{neighbors}
\mathcal{N}(A_i) = \bigcup\big\{ A_j  :  \text{int}\big(\overline{A_i} \cup \overline{A_j}\big) \text{ is connected} \big\}.
\end{equation}
Note that $A_j \subseteq \cN(A_i) \iff A_i \subseteq \cN(A_j)$, so that we may speak of $A_i$ and $A_j$ as being neighbors, which we will denote by $A_i \sim A_j$. 
\end{definition}

Under this definition, in a Euclidean space, balls that only meet at one point are not neighbors, and neither are rectangles in dimension three that intersect only along an edge. Our discussion will be simplified in the case where we consider collections where all sets that intersect are neighbors.  

\begin{definition}[\textbf{Internally connected property}]
\label{def:neighbor_property}
Let $ \cA = \{A_i \}$ be a collection of sets. We say $\cA$ has the internally connected property if
\begin{equation}
\label{eq:pairwise_property}
\overline{A_i} \cup \overline{A_j} \
\text{ connected }\implies \text{int}\big(\overline{A_i} \cup \overline{A_j}\big) \text{ connected }.
\end{equation}
\end{definition}

\figref{neighbors} illustrates these two definitions. 

\section{Axioms}
\label{sec:axioms}
In this section, we develop a definition of the population cluster tree for a  density $f$. Inspired by previous axiomatic approaches to clustering algorithms and in the spirit of Lebesgue integration, we propose a set of axioms for a population cluster tree when the density is piecewise constant with connected support. We then extend this definition to more general densities, and arrive at a definition that is equivalent to Hartigan's tree (\defref{hartigan tree}) for continuous densities with multiple connected components, under some mild assumptions. 

\subsection{Axioms for Piecewise Constant Functions}

Previous work has discussed difficulties in defining what the ``true'' clusters are~\cite{cormack1971review, hartigan1975, hennig2015, vonluxburg2012}, observing that there may not be a single definition for all intents and purposes. 
So as to simplify the situation as much as possible so that a definition may arise as natural, we first consider piecewise constant functions with connected, bounded support. A function in that class  is of the form
\begin{equation}
\label{f constant}
f  = \sum\limits_{i=1}^{m} \lambda_i\, \bbI_{A_i},
\end{equation}  
where, for all $i$, $\lambda_i > 0$ and $A_i$ is a connected, bounded region with connected interior, and we also require that $\supp(f) =  \bigcup_{i=1}^{m} \overline{A_i}$ has connected interior. Additionally, without loss of generality, assume the $A_i$ are disjoint. Let $\cF$ denote the class of all such functions.

\begin{remark}
We require each region $A_i$ and the entire support to not only be connected, but have connected interior, and the same is true of the clusters (Axiom 1). It is well-known that the closure of a connected set is always connected, so that this is a stronger requirement, and is meant to avoid ambiguities. 
\end{remark}

For $f \in \cF$ we propose that a hierarchical clustering $\cC$ should satisfy the following three axioms. For what it's worth, Axiom 1 and Axiom 3 were put forth early on by \citet{Carmichael1968} and, most famously although not as directly, by \citet{hartigan1975}, and also correspond to the 7th item on the list of ``desirable characteristics of clusters'' suggested by \citet{hennig2015}, and Axiom 2 can be motivated by the 13th item on Hennig's list.    
\subsubsection{Axiom 1: Clusters have connected interior} 

\begin{figure}[h!]
    \centering
    \includegraphics[width=0.9\linewidth]{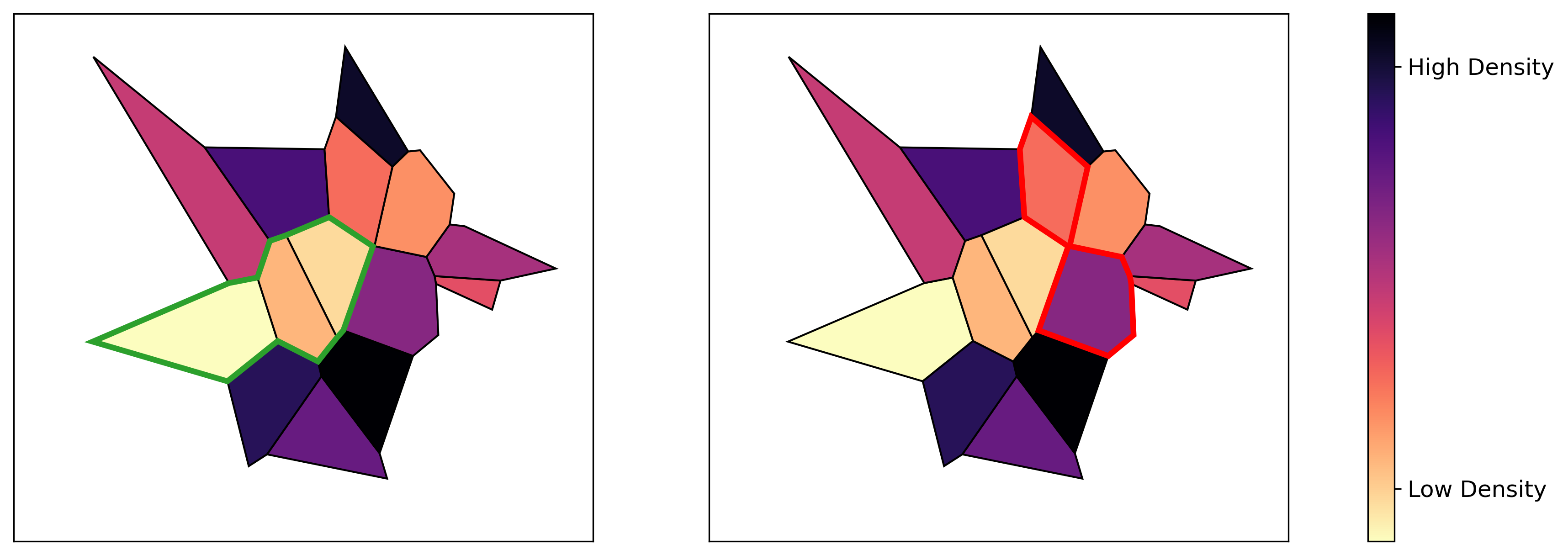}
    \caption{A piecewise constant density in $\cF$. On the left, the highlighted region may be a cluster under Axiom 1 and on the right, the highlighted region is not a cluster under Axiom 1 as the interior is not connected.}
    \label{fig:axiom1}
\end{figure}

\begin{figure}[h!]
    \centering
    \includegraphics[width=0.9\linewidth]{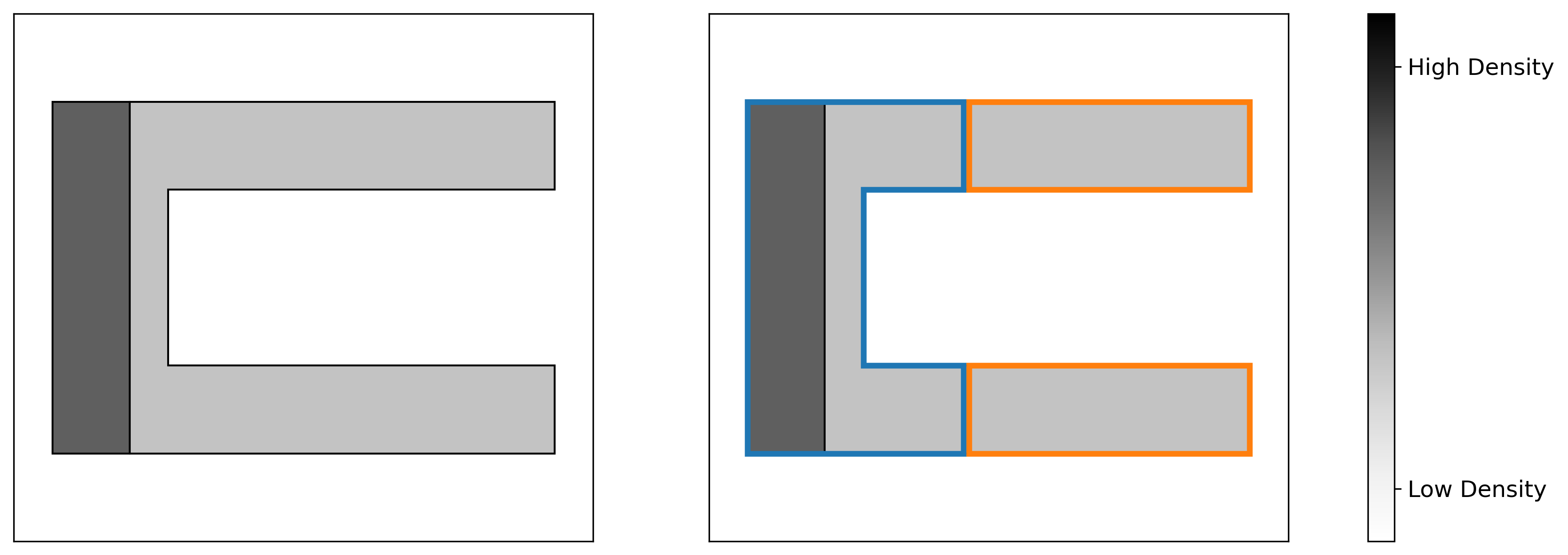}
    \caption{Left: A simple example of a piecewise constant density built on two regions. Right: The clustering output of K-means with number of clusters $K=2$. One of the clusters is disconnected.}
    \label{fig:K-means_disconnected_cluster}
\end{figure}

We propose that any cluster in $\cC$ should not only be connected, but have a connected interior. 
With Axiom 2 below in place, see \eqref{eq:a2}, we may express Axiom~1 as follows:
\begin{gather}
\tag{A1}\label{eq:a1}
\text{If  $C \in \cC$ and $A_i, A_j \subseteq C$, then there are $A_{k_1}, \dots, A_{k_n} \subseteq C$} \\ 
\text{such that $A_i \sim A_{k_1} \sim \cdots A_{k_n} \sim A_j$.}
\end{gather}

For example, for the density in \figref{axiom1}, the highlighted region in the right hand figure should not be a cluster in $\cC$, but the highlighted region in the left hand figure could be a cluster in $\cC$. This reflects the idea that elements of a cluster should in some sense be similar to each other, without imposing additional assumptions on the within-cluster distances, between-cluster distances, the relative sizes of clusters, or the shape of clusters. 

The condition that a cluster be a connected region was considered early on in the literature as it was part of the postulates put forth by \citet{Carmichael1968}. However, it is important to note that this condition is not enforced in other definitions of what a cluster is. Most prominently, K-means can return disconnected clusters --- see \figref{K-means_disconnected_cluster} for an example.

\subsubsection{Axiom 2: Clusters do not partition connected regions of constant density} 

We propose that a connected region with constant density should not be broken up into smaller clusters as this would impose an additional structure that is not present in the density. 
We may write this axiom as:
\begin{equation}
\tag{A2}\label{eq:a2}
\text{Any $C \in \cC$ is of the form $C = \bigcup\limits_{i \in I} A_i$ for some $I \subseteq \{1,2, \dots, m\}$.}
\end{equation}

\figref{axiom2} depicts and example of a valid and invalid cluster under this axiom. Note that as a consequence of this axiom, the within-cluster distances may be larger than the between-cluster distances, depending on the relative widths and separations between regions.

We find this condition to be particularly natural in the present situation where the density is piecewise constant. It is in essence already present in the concept of relatedness introduced by \citet{Carmichael1968}. But it is important to note that other definitions do not enforce this property. This is the case of K-means, which can split connected regions of constant density --- see, again, \figref{K-means_disconnected_cluster} for an example.

\begin{figure}[h!]
    \centering
    \includegraphics[width=0.9\linewidth]{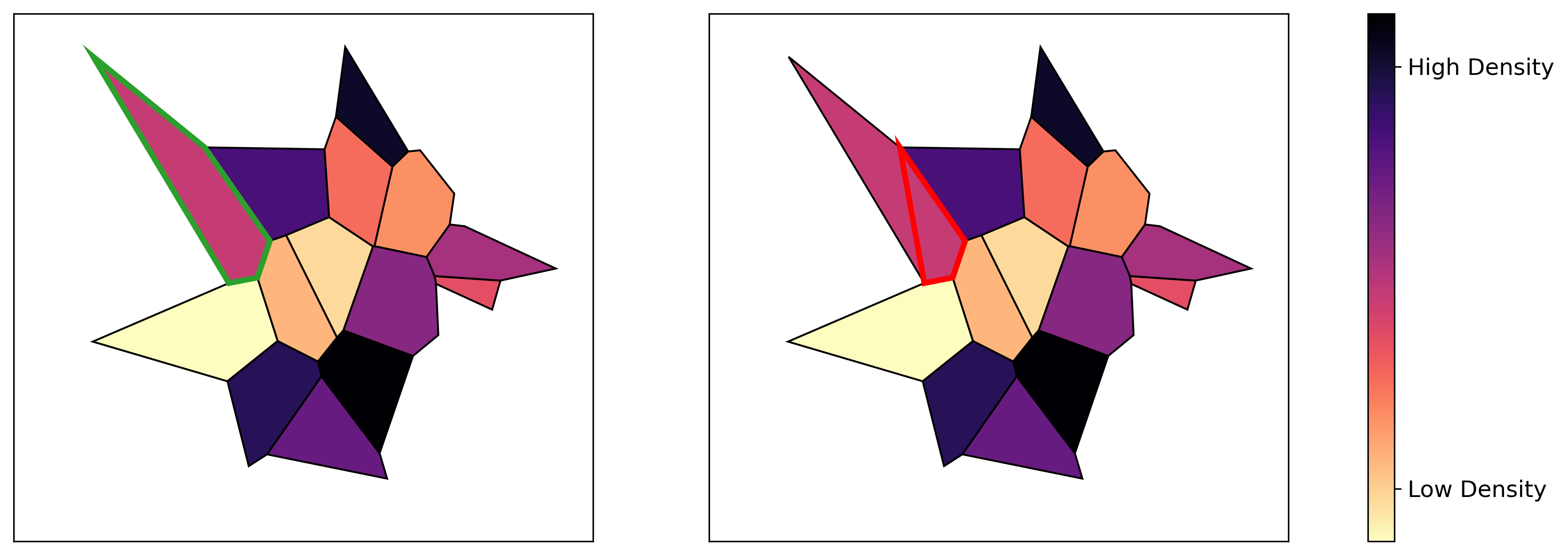}
    \caption{On the left, the highlighted region could be a cluster under Axiom 2, but the highlighted region on the right oversegments a region of constant density, and should not be a cluster. }
    \label{fig:axiom2}
\end{figure}

\begin{figure}[h!]
    \centering
    \includegraphics[width=0.9\linewidth]{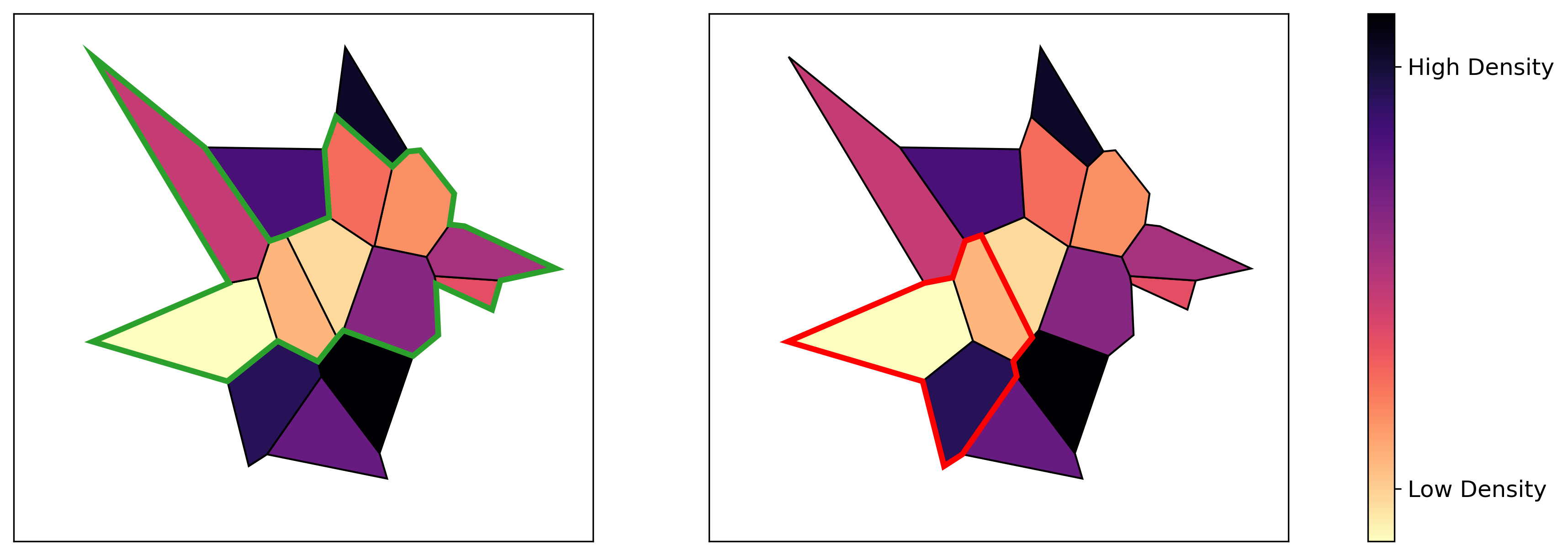}
    \caption{On the left, the lowest density in highlighted cluster exceeds the largest density in a neighboring set. On the right, the highlighted cluster contains a region with lower density than a neighbor, and thus this should not be a cluster.}
    \label{fig:axiom3}
\end{figure}

\subsubsection{Axiom 3: Clusters are surrounded by regions of lower density} 

We propose that a cluster should be surrounded by regions of lower density, meaning that:
\begin{equation}
\tag{A3}\label{eq:a3}
\text{For any $C \in \cC$, it holds that $\inf\limits_{x \in C } f(x) > \!\!\!\!\!\sup_{x \in \cN(C) \setminus C} f(x)$,} 
\end{equation}
where, if $C = \bigcup_{i \in I} A_i$, then $\cN(C) = \bigcup_{i \in I} \cN(A_i)$ denotes the  neighbor of $C$, extending the definition given in \eqref{neighbors}. \figref{axiom3} includes an example. 

This is one of the postulates of  \citet{Carmichael1968}, although it was perhaps most popularized by Hartigan in his well-known book \cite[Ch 11]{hartigan1975}. Although it is not part of most other approaches to clustering --- K-means being among those as \figref{K-means_disconnected_cluster} shows --- we find that this condition is rather compatible with the colloquial understanding of `points clustering together'.

\subsection{Finest hierarchical clustering}

\begin{definition}[\textbf{Finer cluster tree}]
We say that a cluster tree $\cC$ is finer than (or a refinement of) another cluster tree $\cC'$ if $\cC$ includes all the clusters of $\cC'$, namely, $C \in \cC' \implies C \in \cC$.
\end{definition}

As it turns out, given a nonnegative function, there is one, and only one, finest cluster tree among those satisfying the axioms above.

\begin{proposition}
\label{prp:axiom_uniqueness}
For any $f \in \cF$, there exists a unique finest hierarchical clustering of $f$ among those satisfying the axioms.
\end{proposition}

\begin{proof} 
Let $f$ be as in \eqref{f constant}.
The proof is by construction. 
Let $\cC_*$ denote the collection of every cluster that satisfies \eqref{eq:a1}, \eqref{eq:a2}, and \eqref{eq:a3}.
Clearly, it suffices to show that $\cC_*$ is a hierarchical clustering (\defref{cluster tree}). 
%First, $\cC_*$ contains $\supp(f)$, which satisfies \eqref{eq:a1} by assumption that it is connected,  \eqref{eq:a2} by definition, and \eqref{eq:a3} because the maximum there is over an empty set and is therefore $= -\infty$.
Take two clusters in $\cC_*$, say $C_1 = \bigcup_{i \in I_1} A_i$ and $C_2 = \bigcup_{i \in I_2} A_i$.  We need to show that $C_1$ and $C_2$ are either disjoint or nested.
Suppose for contradiction that this is not the case, so that $C_1$ and $C_2$ are neither disjoint nor nested. Since they are not disjoint, there is $i \in I_1 \cap I_2$, so that $A_{i} \subseteq C_1 \cap C_2$. And since they are disjoint, there is $j \in I_1 \setminus I_2$, so that $A_{j} \subseteq C_1 \setminus C_2$.
By \eqref{eq:a1}, there are $i_1, \dots, i_s \in I_1$ such that $A_i \sim A_{i_1} \sim \cdots \sim A_{i_s} \sim A_{j}$. Let $t = \max\{q : A_{i_q} \subseteq C_2\}$, so that $A_{i_t} \subseteq C_2$ while while $A_{i_{t+1}} \nsubseteq C_2$, and in particular $A_{i_{t+1}} \subseteq \cN(C_2) \setminus C_2$, and applying \eqref{eq:a3}, we get

\[
\min_{C_2} f
> \max_{\cN(C_2) \setminus C_2} f
\ge \lambda_{i_{t+1}}
\ge \min_{C_1} f,\]
using at the end the fact that $A_{i_{t+1}} \subseteq C_1$.
However, we could also get the reverse inequality, $\min_{C_1} f > \min_{C_2} f$, in the same exact way, which would result in a contradiction. 
\end{proof}

\prpref{axiom_uniqueness} justifies the following definition.

\begin{definition}[\textbf{Finest axiom cluster tree}]
\label{def:axiomatic_cluster_tree}
For $f \in \cF$, we denote by $\cC_f^{*}$ the finest cluster tree of $f$ among those satisfying the axioms. 
\end{definition}

\subsection{Comparison with Hartigan's Cluster Tree}

It is natural to compare the finest axiom cluster tree of \defref{axiomatic_cluster_tree} with the Hartigan cluster tree of \defref{hartigan tree}. 
First, observe that for  $f \in \cF$, $\cH_f$ satisfies \eqref{eq:a2} and \eqref{eq:a3}. However, $\cH_f$ need not satisfy \eqref{eq:a1}, as clusters in $\cH_f$ are only required to be connected. As a result, in general, the Hartigan tree $\cH_f$ is not the same as the finest axiom cluster tree $\cC^*_f$.  A counter example is given in~\figref{Hartigan}.

We define $\cF_{\rm int}$ as the class of functions in $\cF$ with $\{ A_i \}$ in \eqref{f constant} having the internally connected property (\defref{neighbor_property}). 

\begin{theorem}
\label{thm:hartigan_equivalence_pw}
For any $f \in \cF_{\rm int}$, it holds that $\cC^*_f = \cH_f$. 
\end{theorem}

\begin{proof}
First, observe that under our assumption, $\cH_f$ satisfies all axioms \eqref{eq:a1}, \eqref{eq:a2}, and \eqref{eq:a3}. Thus, because $\cC^*_f$ is the finest cluster tree among those satisfying the axioms (\prpref{axiom_uniqueness}), it must be the case that $\cH_f \subseteq \cC^*_f$.

For the reverse inclusion, take any $C \in \cC^*_f$. We want to show that $C \in \cH_f$.
Recalling the definition of $h_f$ in \eqref{height}, define $\lambda = h_f(C)$ and let $M$ be the maximally connected subset of $\{f \geq \lambda \}$ that contains $C$. We need to show that $C = M$. 
Noting that $C$ is of the form $\bigcup_{i \in I} A_i$ because of \eqref{eq:a2}, and that $M$ must be of the form $\bigcup_{j \in J} A_j$ because $f$ is of the form \eqref{f constant}, and that $M$ contains $C$ by definition, it is the case that $I \subseteq J$.

Suppose for contradiction that $C \ne M$, so that $I \ne J$. Since $M$ is connected, there must be $A_i$ in $C$ and $A_j$ in $M \setminus C$ such that $A_i \cup A_j$ is connected. As is well-known, this implies that $\overline{A_i \cup A_j} = \overline{A_i} \cup \overline{A_j}$ is connected, and since $f \in \cF_{\rm int}$, $\text{int}(\overline{A_j} \cup \overline{A_i})$ is also connected, in turn implying that $A_i \sim A_j$. Applying \eqref{eq:a3}, we get that $\lambda > f(A_j)$, and this is a contradiction since $A_j \subseteq M$ and $M$ is part of the upper $\lambda$-level set. 
\end{proof}

\begin{figure}[h!]
    \centering
    \includegraphics[width=0.5\linewidth]{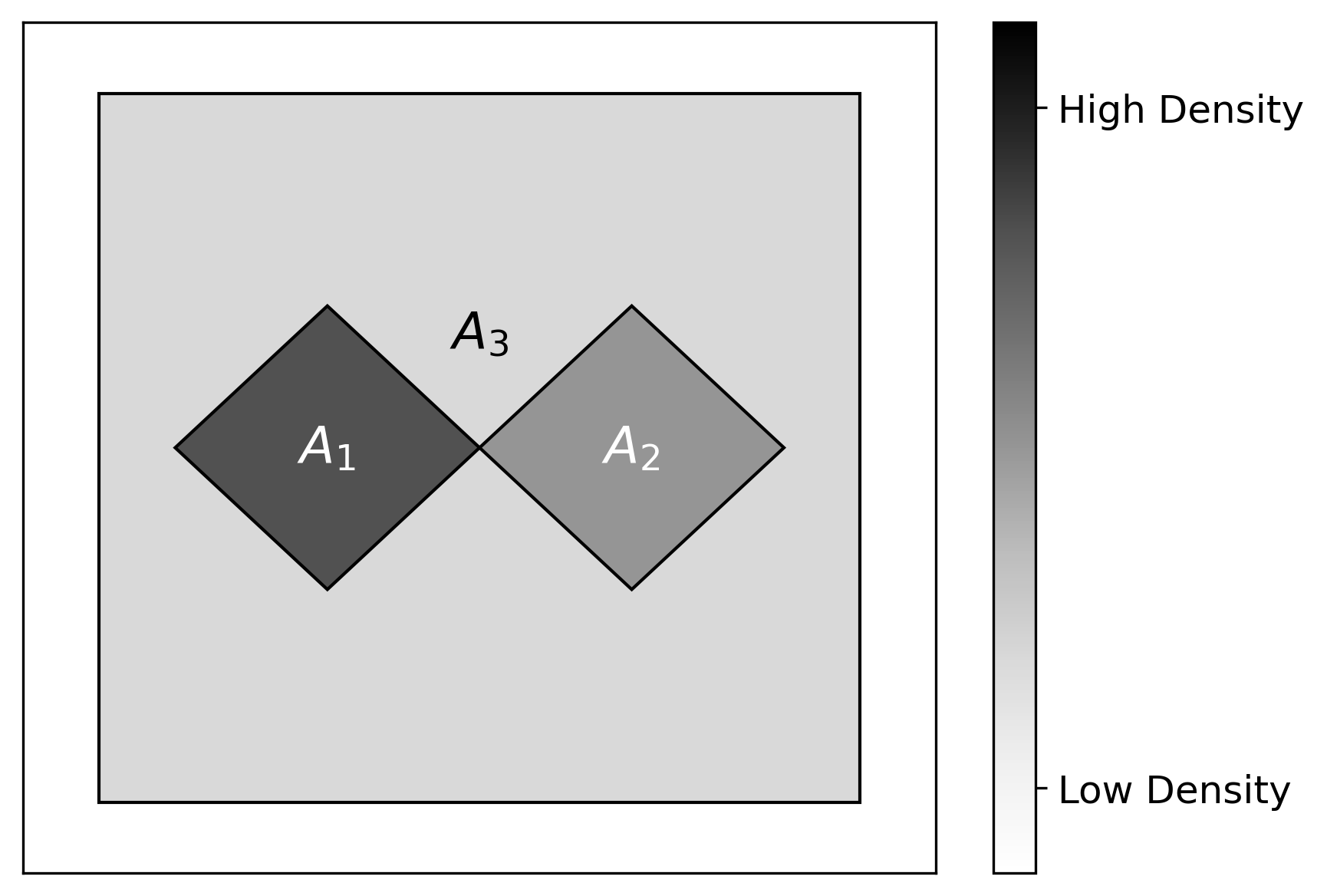}
    \caption{An example where Hartigan's cluster tree does not satisfy the axioms, so that $\cC^*_f \neq \cH_f$. Indeed, $\cH_f = \{ A_1, A_1 \cup A_2, A_1 \cup A_2 \cup A_3\}$ but $A_1 \cup A_2 \notin \cC^*_f$ because $\text{int}(A_1 \cup A_2)$ is not connected. Instead, we have $\cC^*_f = \{ A_1, A_2, A_1 \cup A_2 \cup A_3\}$.}
    \label{fig:Hartigan}
\end{figure}

\begin{remark}
\label{rem:relax}
As a relaxation of Axiom 1, we could simply require a cluster to be connected, and allow it to have disconnected interior. If the definition of a neighboring region were also relaxed so that if the closure of the union of two sets is connected, then the sets are considered neighbors, then the relaxed Axiom 1, original Axiom 2, and original Axiom 3 would yield an axiomatic definition of a cluster tree that is identical to the Hartigan tree for $f \in \cF$. All that said, we find the requirement that the interior be connected in our original Axiom 1 (and in \defref{neighbors}) to be more natural and robust.
\end{remark}

\section{Extension to Continuous Functions}
\label{sec:continuous}
Having defined the finest axiom cluster tree for a piecewise constant function (\defref{axiomatic_cluster_tree}), we now examine its implication when piecewise constant functions are used to approximate continuous functions. More specifically, we consider sequences of piecewise constant functions in $\cF_{\rm int}$ converging to a continuous function, and show that, under some conditions, the corresponding finest axiom cluster trees converge to the Hartigan cluster tree of the limit function in merge distortion metric (\defref{merge metric}).

\subsection{Functions with connected support}
\label{sec:continuous_connected}

We start with continuous functions whose support has connected interior. 

\begin{definition}
%[\textbf{finest hierarchical clustering for continuous functions}]
\label{def:finest cluster tree continuous}
Given a continuous function $f$ with connected support, we say that $\cC$ is an axiom cluster tree for $f$ if there is a sequence $(f_n) \subseteq \cF_{\rm int}$ that uniformly approximates $f$ such that  
\begin{equation}
\label{eq:continuous_extension}
\lim\limits_{n \rightarrow \infty} d_M((\cC^*_{f_n}, h_{f_n}), (\cC, h_f)) = 0.
\end{equation}
\end{definition}

At this point it is not clear whether a continuous function admits an axiom cluster tree. However, if it does, then its Hartigan cluster tree is one of them and, moreover, all other axiom cluster trees are zero merge distortion distance away.

\begin{theorem}
\label{thm:hartigan_equivalence}
Suppose $f$ is a continuous function that admits an axiom cluster tree. Then its Hartigan tree $\cH_f$ is an axiom cluster tree for $f$. Moreover, if $\cC$ is an axiom cluster tree for $f$, then $d_M((\cC, h_f), (\cH_f, h_f)) = 0$.
\end{theorem}

\begin{proof}
Let $\cC$ be an axiom cluster tree for $f$. 
By \defref{finest cluster tree continuous}, there is a sequence $(f_n)$ in $\cF_{\rm int}$ that converges uniformly to $f$ such that \eqref{eq:continuous_extension} holds. 
By the triangle inequality,
\begin{align}
d_M((\cC, h_f), (\cH_f, h_f))
&\le d_M((\cC, h_f), (\cC^*_{f_n}, h_{f_n})) + d_M((\cC^*_{f_n}, h_{f_n}), (\cH_{f}, h_f)).
\end{align}
We already know that the first term on the RHS tends to zero.
For the second term, using \thmref{hartigan_equivalence_pw} and \lemref{th17}, 
\begin{align}
\label{hartigan_equivalence_proof1}
d_M((\cC^*_{f_n}, h_{f_n}), (\cH_{f}, h_f))
&= d_M((\cH_{f_n}, h_{f_n}), (\cH_{f}, h_f)) 
\leq  \|f_n - f\|_{\infty}\rightarrow 0, \quad n \to \infty.
\end{align}
We thus have that $d_M((\cC, h_f), (\cH_f, h_f)) = 0$ --- this being true for any axiom cluster tree $\cC$. In the process, we have also shown in \eqref{hartigan_equivalence_proof1} that $\cH_f$ is axiomatic.
\end{proof}

The remaining of this subsection is dedicated to providing sufficient conditions on a function $f$ for the existence of sequence $(f_n) \subseteq \cF_{\rm int}$ that converges uniformly to $f$. In formalizing this, we will utilize the following terminology and results. 

\begin{definition}[\textbf{Internally connected partition property}] 
We say that $\Omega$ has the internally connected partition property if it is connected, and for any $r > 0$, there exists a locally finite partition $\{A_i\}$ of $\Omega$ that has the internally connected property and is such that, for all $i$, $A_i$ is connected with connected interior and diameter at most~$r$.
\end{definition}

We establish in \prpref{euclidean} that any Euclidean space (and, consequently, of any finite-dimensional normed space) has the internally connected partition property. And we conjecture that this extends to some Riemannian manifolds. 

\begin{proposition}
\label{prp:uniform_approx_conditions}
Suppose $(\Omega, d)$ is a metric space where all closed and bounded sets are compact\footnote{This is sometimes called the Heine--Borel property.}, and that has the internally connected partition property. Let $f: \Omega \rightarrow [0, \infty)$ be continuous with all upper level sets bounded, and such that the upper $\lambda$-level set is connected when $\lambda > 0$ is small enough.
%\begin{equation}
%\label{uniform_approx_conditions1}
%\text{$\{x : f(x) \geq \lambda\}$ has the internally connected partition property when $\lambda > 0$ small enough.}
%\end{equation}
Then, there is a  sequence $(f_n) \in \cF_{\rm int}$ that converges uniformly to $f$.  
\end{proposition}

\begin{proof} 
It is enough to show that, for any $\eta > 0$, there is a function in $\cF_{\rm int}$ within $\eta$ of $f$ in supnorm. 
Therefore, fix $\eta > 0$, and take it small enough that the upper $\eta$-level set, $K = \{x : f(x) \geq \eta\}$, is connected.
Consider 
\begin{equation}
\label{tube}
K_1 = \big\{y : d(y, x) \le 1, \text{ for some } x \in K\big\}.
\end{equation}
In particular, $K_1$ is compact, and since $f$ is continuous on $K_1$, it is uniformly so, and therefore there exists $0 < \eps < 1$ such that, if $x,y \in K_1$ are such that $d(x,y) \le \eps$, then $|f(x) - f(y)| \le \eta$. 

By the fact that $\Omega$ has the internally connected partition property, it admits a locally finite partition $\{A_i\}$ with the internally connected property and such that, for all $i$, $A_i$ has connected interior and diameter at most~$\eps$.
Let 
\[I = \{i : A_i \cap K \neq \emptyset\},\]
and note that $I$ is finite and that $K \subseteq \bigcup_{i \in I} A_i \subseteq K_1$.  
For $i \in I$, let $\lambda_i = \sup_{x \in A_i} f(x)$. Because $A_i \cap K \ne \emptyset$, we have $\lambda_i \geq \eta$. 
Finally, we define the piecewise constant function $g = \sum_{i\in I} \lambda_i \bbI_{A_i}$.
We claim that $g \in \cF_{\rm int}$.
Since $\{A_i : i \in I\}$ inherits the internally connected property from $\{A_i\}$, all we need to check is that $\bigcup_{i \in I} \overline{A_i}$ is connected. To see this, first note that it is enough that $\bigcup_{i \in I} A_i$ be connected (since the closure of a connected set is connected). Suppose for contradiction that $\bigcup_{i \in I} A_i$ is disconnected, so that we can write it as a disjoint union of $\bigcup_{i \in I_1} A_i$ and $\bigcup_{i \in I_2} A_i$, where $I_1$ and $I_2$ are non-empty disjoint subsets of $I$. Because $K \subseteq \bigcup_{i \in I} A_i$, then we have that $K$ is the disjoint union of $K_1 := \bigcup_{i \in I_1} A_i$ and $K_2 := \bigcup_{i \in I_2} A_i$, both non-empty by construction, so that $K$ is not connected --- a contradiction. 

We now show that $\|f-g\|_\infty \le \eta$. 
For $x \notin \bigcup_{i\in I} A_i$, $g(x) =0$ and since $x \notin K$, $f(x) < \eta$, so that $|f(x)-g(x)| \le \eta$.
For $x \in A_i$, for some $i \in I$, $g(x) = f(y)$ for some $y \in \overline{A_i}$, and because $x,y \in K_1$ and $d(x,y) \leq \eps$, we have $|f(y) - f(x)| \leq \eta$. 
\end{proof}

\subsection{Functions with Disconnected Support}
\label{sec:more than two}

So far, we have focused our attention on densities whose support has connected interior. However, there is no real difficulty in extending our approach to more general densities. Indeed, given a function with support having disconnected interior, our approach can define a hierarchical clustering of each connected component of $\{f > 0\}$.
%\footnote{Each connected component of $\{f>0\}$ is open if $f$ is continuous as we have assumed that $(\Omega, d)$ is locally connected.}  

In more detail, let $f$ be a function of the form 
\begin{equation}
\label{more than two}
f = \sum_{j=1}^{N} f_j,
\end{equation}
where $\text{int}(\supp(f_j)) \cap \text{int}(\supp(f_k)) = \emptyset$ when $j \ne k$. First, suppose that each $f_j \in \cF$. If we apply the axioms of \secref{axioms}, we obtain that $C$ is a cluster for $f$ if and only if it is a cluster for one of the $f_j$, and consequently that the finest axiom cluster tree for $f$ is simply the union of the finest axiom cluster trees for the $f_j$, i.e.,
\[\cC^*_f = \bigcup_{j=1}^N \cC^*_{f_j}.\]
If $f$ is continuous, that is, if each $f_j$ in \eqref{more than two} is continuous, we may proceed exactly as in \secref{continuous_connected} and, based on the facts that $\cH_f = \bigcup_j \cH_{f_j}$, $h_f(C) = h_{f_j}(C)$ when $C \in \cC_{f_j}$, and 
\[d_M((\cC, h_f), (\cC', h_f)) = \max_{j = 1, \dots, N} d_M((\cC_j, h_{f_j}), (\cC'_j, h_{f_j})),\]
for any two axiom cluster trees for $f$, $\cC = \bigcup_j \cC_j$ and $\cC' = \bigcup_j \cC'_j$ (all cluster trees for $f$ are of that form), we find that \thmref{hartigan_equivalence} applies verbatim. 

This is as far as our approach goes. The end result is Hartigan's cluster tree, with the same caveats that come from using the merge distortion metric detailed in \secref{merge}. In particular, instead of a tree we have a forest with $N$ trees in general, one for each $f_j$. We find this end result natural, but if it is desired to further group regions (see \figref{multiple_components} for an illustration), one possibility is to apply a form of agglomerative hierarchical clustering to the `clusters', $\supp(f_1), \dots, \supp(f_N)$. (In our definition, these are not clusters of $\cH_f$, but this is immaterial.) Doing this presents the usual question of what clustering procedure to use, but given what we discuss in \secref{algorithm_discussion}, single-linkage clustering would be a very natural choice.

\begin{figure*}[h!]
\centering
\includegraphics[width=.6\linewidth]{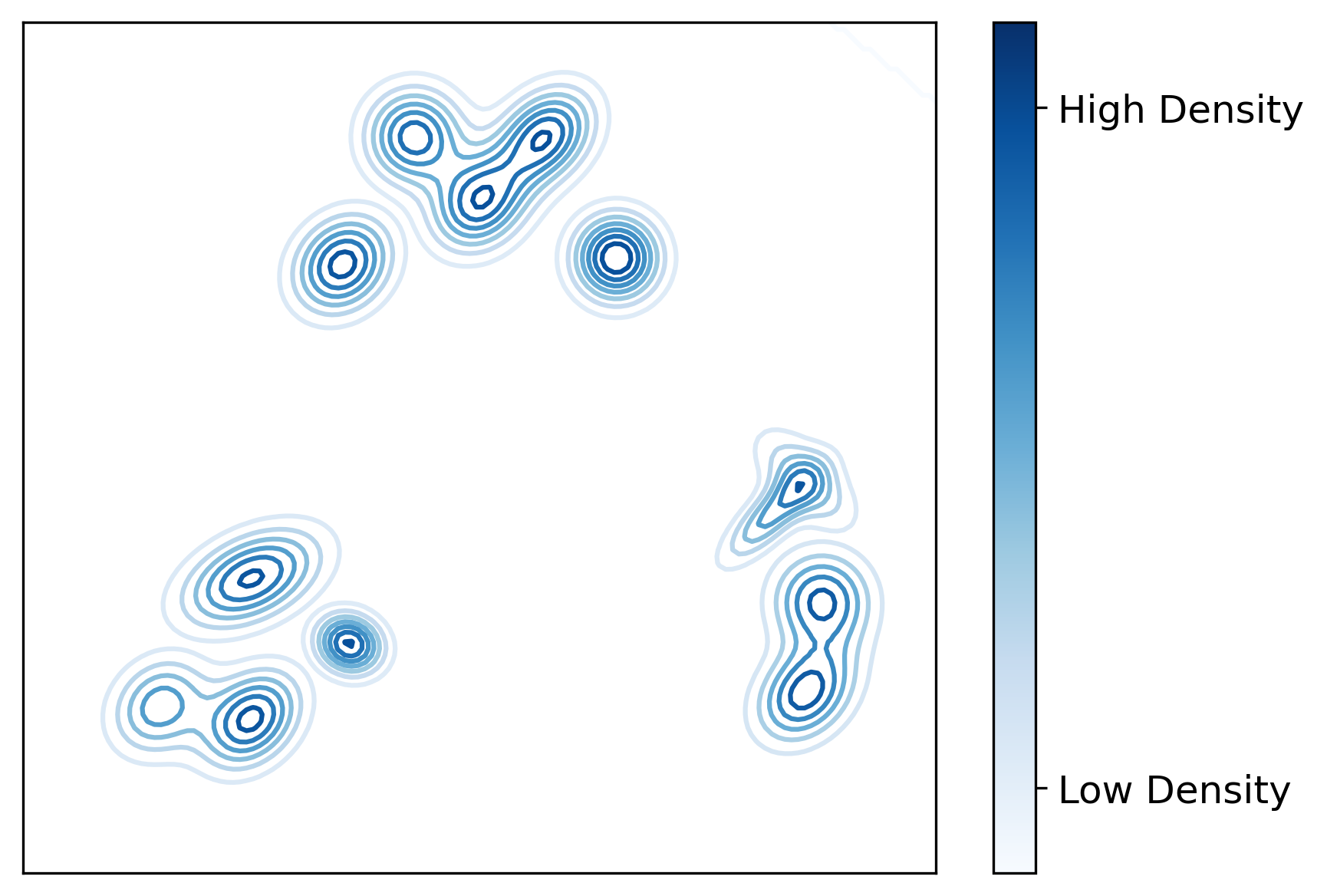}
\caption{An example of density with a support that has disconnected interior which appears to exhibit some clustering structure beyond that happening inside each of its eight connected components.}
\label{fig:multiple_components}
\end{figure*}

\section{Discussion}
\label{sec:discussion}

\subsection{Extensions}

%\subsubsection{Piecewise Continuous Functions}
We speculate that our axiomatic definition of hierarchical clustering can be extended beyond continuous functions (\secref{continuous}) to piecewise continuous functions with connected support by the same process of taking a limit of sequences in $\cF_{\rm int}$ that uniformly approximate the function of interest~$f$. 

The natural approach is to work within each region where $f$ is continuous, say $R$, and to consider there a partition of $R$ that would allow the definition of a piecewise constant function approximating $f$ uniformly on $R$.
The main technical hurdle is the construction of such a partition with the internally connected property, as a region $R$ may not be regular enough to allow for that. Additionally, even if there is a partition with the internally connected property on each region, taken together, these partitions may not have the internally connected property. We see some possible workarounds, but their implementation may be complicated. 

\subsection{Practical Implications}
We first examine some implications of adopting the axioms defining clusters in \secref{axioms}.

\subsubsection{Algorithms}
\label{sec:algorithm_discussion} 
A large majority of existing approaches to clustering return clusters that, when taken to the large-sample limit, do not necessarily satisfy the proposed axioms. This is true of K-means and all agglomerative hierarchical clustering that we know of, with the partial exception of single-linkage clustering, as repeatedly pointed out by \citet{hartigan1977, hartigan1981, hartigan1985}. Interestingly, single-linkage clustering arises out of various axiomatic discussion of (flat) clustering such as \cite{kleinberg2002, ben2008, zadeh2009, Cohen-Addad2018}, and also of hierarchical clustering \cite{jardine1968, carlsson2010}. 

This is despite the heavy criticism of single linkage in the literature for its ``chaining'' tendencies.
However, in practice ``chaining'' can be a concern, and regularized variants of single-linkage clustering are preferred. Most prominently, this includes DBSCAN~\cite{ester1996}, which has been shown to consistently estimate the Hartigan cluster tree \cite{wang2019} in the merge distortion metric when the underlying density is H\"{o}lder smooth, for example; see, also,  the ``robust'' variant of single-linkage clustering proposed in \cite{eldridge2015}, also shown to be consistent under some conditions.

\subsubsection{Clustering in High Dimensions}
\citet{wang2019} derive minimax rates for the estimation of the Hartigan cluster tree, which turn out to match the corresponding minimax rates for density estimation in the $L_{\infty}$ norm under assumptions of H\"{o}lder smoothness on the density. In particular, these rates exhibit the usual behavior in that they require that the sample size grow exponentially with the dimension. This is a real limitation of adopting the definition of cluster and cluster tree that we proposed in \secref{axioms}, although the usual caveats apply in that the curse of dimensionality is with respect to the intrinsic dimension if the density is in fact with respect to a measure supported on a lower-dimensional manifold~\cite{balakrishnan2013}; and `assuming' more structure can help circumvent the curse of dimensionality, as done for example in \cite{chacon2019mixture}, where a mixture is fitted to the data before applying modal clustering.

\subsection{Axiomatic Definition of Flat Clustering}

We have already mentioned some axiomatic approaches to defining flat \cite{ben2008, kleinberg2002, zadeh2009, puzicha2000, strazzeri2022, Cohen-Addad2018} and hierarchical \cite{jardine1968, jardine1967structure, SIBSON1970405, carlsson2010} clustering algorithms. 
But beyond these efforts, defining what clusters are has proven to be challenging for decades, in particular due to the fact that the problem is at the very core of Taxonomy. In his comprehensive review of the field at the time, \citet{cormack1971review} says that ``There are many intuitive ideas, often conflicting, of what constitutes a cluster, but few formal definitions.''
More recent discussions include those of \citet{estivill2002so}, \citet{vonluxburg2012} and that of \citet{hennig2015}.
As \citet{alma9918936281206531} say in their recent book on clustering, ``The clustering problem has been addressed extensively, although there is no uniform definition for data clustering and there may never be one''.

By focusing on hierarchical clustering, as others have done (e.g., \cite{carlsson2010}), we circumvented the thorny issue of defining the correct number of clusters, and propose simple axioms defining a cluster that are `natural' in our view --- in fact, as we pointed out earlier, the axioms we propose are hardly novel. However, the question of an axiomatic definition of a flat clustering of a population or density support remains intriguing, and we hope to address it in future work. For now, we remark that the definition most congruent with our definition of hierarchical clustering is that of \citet{fukunaga1975}, which when the density $f$ has some regularity amounts to partitioning $\supp(f)$ according to the basin of attraction of the gradient ascent flow defined by $f$. This has been argued in recent work \cite{arias2023,arias2023b}. It would be interesting to see whether one can arrive at this definition of clustering by the proposal of a `natural' set of axioms.

% \subsection*{Acknowledgements} 
%This work was partially supported by the US National Science Foundation (DMS 1916071).

%\printbibliography
\bibliographystyle{abbrvnat}
\bibliography{ref}

\newpage
\appendix
\label{appendix:Appendix}

\section{Merge Distortion Metric}
\label{sec:merge}

In this section we discuss some limitations and issues of the merge distortion metric. We restrict our attention to the situation considered in~\cite{eldridge2015} where the height of a tree is defined by the density itself as in \eqref{height}. We denote the density by $f$ and the corresponding height function by $h$, and we identify a cluster tree $\cC$ with the dendrogram $(\cC, h)$ whenever needed. We only consider cluster trees $\cC$ made of clusters $C \in \cC$ satisfying $h(C) > 0$. 
Our discussion applies to non-negative functions, and throughout this section, $f$ will be non-negative.

The main issue that we want to highlight is that the merge distortion metric is only a pseudometric, and not a metric, on general cluster trees, as it is possible to have $d_M(\cC, \cC') = 0$ even when $\cC$ and $\cC'$ are not isomorphic. (To be clear, we take the partially ordered sets $\cC$ and $\cC'$ to be isomorphic if they are order isomorphic.) Two examples of this follow. 

\begin{example}
\label{ex:merge1}
Consider $f = \frac{1}{2}\bbI_{A_1} + \frac{1}{3}\bbI_{A_2} + \frac{1}{6}\bbI_{A_3}$ where the $A_i$ are disjoint sets with unit measure. Let  $\cC = \{A_1, A_1 \cup A_2, A_1 \cup A_2 \cup A_3 \}$ and $\cC' = \{A_1, A_2, A_1 \cup A_2, A_1 \cup A_2 \cup A_3 \}$. Both $\cC$ and $\cC'$ are cluster trees and it can be checked that $m_{\cC}(x,y) = m_{\cC'}(x,y)$ for all $x,y$ so that $d_M(\cC, \cC') = 0 $. However, the two trees are clearly not isomorphic. 
\end{example}

\begin{example}
Consider $f = \bbI_{A}$ where $A$ has unit measure. Then any collection of subsets of $A$ with a nested structure is a cluster tree for $f$, and the merge distortion distance between any pair of such cluster trees is zero. 
\end{example}

The issue in the preceding examples arises because a cluster tree contains nested clusters with the same cluster height. For example, in Example \ref{ex:merge1}, the addition of the cluster $A_2$ to $\cC$ does not change the merge height of any two points, and hence the merge distance between $\cC$ and $\cC'$ is zero.

Note that neither of these examples compare Hartigan trees, and we suspect in the original merge distortion metric paper~\cite{eldridge2015}, the claim (without proof) that if the merge distortion metric is zero then the trees must be isomorphic was intended in the context of comparing Hartigan trees. This is true for comparing Hartigan trees of continuous densities on $\bbR^d$, as for Hartigan trees of continuous functions $f,g$,  
\begin{equation}
\label{eq:merge_sup}
d_M(\cH_f, \cH_g) = \|f-g\|_{\infty}.
\end{equation}
This is established in \cite[lem 1]{kim2016}.
The proof of that result can be adapted to extend the result to the case where $f$ is continuous and $g$ is piecewise-continuous satisfying an additional regularity condition that, for every $x$ in its support, there exists a $\delta$ small enough such that $g$ is continuous on a half-ball centered at $x$ of radius $\delta$.

In view of \thmref{hartigan_equivalence}, we are particularly interested in understanding how different a cluster tree $\cC$ such that $d_M(\cC, \cH_f) = 0$ can be from $\cH_f$. 
The following results clarify the situation. 
The $\lambda$-level set of $f$ is defined as
\begin{equation}
\level_\lambda = \{f = \lambda\}.
\end{equation}

\begin{proposition}
\label{prp:merge zero sufficient} 
Let $f$ be a continuous density. 
Consider a collection of clusters of the form
\begin{equation}
\label{merge zero sufficient} 
\cC = \left( \cH_f \setminus \{C_i : i \in I\} \right) \cup \{\cS_j : j \in J\},
\end{equation}
where $C_i \in \text{cc}(\up_{\lambda_i})$ for some $\lambda_i > 0$ such that $\{\lambda_i : i \in I\}$ has empty interior; and $\cS_j$ is a cluster tree of $\level_{\lambda_j}$ for some $\lambda_j > 0$ such that $\{\lambda_j : j \in J\}$ are all distinct. Then $\cC$ is a cluster tree for $f$ satisfying $d_M(\cC, \cH_f) = 0$. 
\end{proposition}

\begin{proof}
We will use the fact that, by continuity of $f$, the supremum in \eqref{merge height} is attained, or more specifically, that if $\lambda = m_f(x,y)$, there is a connected component $C$ of $\up_\lambda$ that contains $x$ and $y$. The continuity of $f$ also implies that, for any subset $C$, $h(C) = h(\overline{C})$.

We first show that any $\cC$ defined as in \eqref{merge zero sufficient} is a cluster tree. 
Indeed, the removal of any number of clusters preserves the nested structure. Now, consider adding $\cS_j$, a cluster tree for $\level_{\lambda_j}$ for some $\lambda_j > 0$. We may clearly assume that $\cS_j$ is a cluster tree for a connected component of $\level_{\lambda_j}$, say $B_j$, which is itself contained in some $C_j \in \text{cc}(\up_{\lambda_j})$, so that $S \subseteq B_j \subseteq C_j$ for any $S \in \cS_j$. 
Take $C \in \cH_f$ distinct from $C_j$. 
We show that either $S \cap C = \emptyset$ or $S \subseteq C$ for any $S \in \cS_j$.
Let $\lambda = h(C)$ so that $C$ is a connected component of $\up_{\lambda}$. 
If $\lambda = \lambda_j$, $C_j$ and $C$ are disjoint. 
If $\lambda < \lambda_j$, $B_j$ is disjoint from $C$ unless $C$ contains $C_j$. If this is the case, $C$ also contains $B_j$, and therefore $S$. 
If $\lambda > \lambda_j$, $B_j \subseteq \level_{\lambda_j}$, $C \subseteq \up_{\lambda}$, and $\level_{\lambda_j} \cap \up_{\lambda} = \emptyset$.
Take $S' \in \cS_k$. We show that $S$ and $S'$ are either disjoint or nested. This is the case if $j = k$ by assumption that $\cS_j$ is a cluster tree. For $j \ne k$, $B_j$ and $B_k$ are disjoint since, by assumption, $\lambda_j \ne \lambda_k$ in that case. (We have used the fact that two distinct clusters in $\cH_f$ have disjoint boundaries.) 

To go further, we use the assumption that $\Lambda := \{\lambda_i : i \in I\}$ has empty interior. 
We want to show that $m_\cC(x,y) = m_f(x,y)$ for any pair of points $x$ and $y$. 
First, consider $\cC_1 = \cH_f \setminus \{\text{cc}(\up_{\lambda_i}) : i \in I\}$. Clearly, because the merge height is defined based on a supremum, $m_{\cC_1}(x,y) \le m_\cC(x,y)$. 
Let $\lambda = m_f(x,y)$, so that there is $C \in \text{cc}(\up_\lambda)$ such that $x,y \in C$. 
If $\lambda \ne \lambda_i$ for all $i \in I$, then $m_{\cC_1}(x,y) \ge \lambda$. 
If $\lambda = \lambda_i$ for some $i \in I$, we reason as follows. For $t < \lambda$, let $C_t$ be the connected component of $\up_t$ that contains $C$. Then $x,y \in C_t$ for all $t < \lambda$, and therefore $m_{\cC_1}(x,y) \ge t$ for any $t < \lambda$ not in $\Lambda$. Since $\Lambda$ has empty interior, its complement is dense in $\Lambda$, and by continuity of $f$ this implies that $m_{\cC_1}(x,y) \ge \lambda$.
We have thus established that $m_{\cC_1}(x,y) \ge \lambda = m_f(x,y)$, which then implies $m_{\cC}(x,y) \ge m_f(x,y)$.
Next, consider $\cC_2 = \cH_f \cup \{\cS_j : i \in J\}$, so that $m_{\cC_2}(x,y) \ge m_\cC(x,y)$. 
Consider $S \in \cS_j$, so that $S \subseteq C_j$ for some $C_j \in \text{cc}(\up_{\lambda_j})$.
Because $h(S) \le h(C_j)$ and $C_j \in \cH_f$, the merge height of $x$ and $y$ is not increased by adding $S$ to $\cH_f$. Therefore, $m_{\cC_2}(x,y) \le m_f(x,y)$, which then implies that $m_{\cC}(x,y) \le m_f(x,y)$.
\end{proof}

It turns out that the condition \eqref{merge zero sufficient} is not necessary for a cluster tree $\cC$ to satisfy $d_M(\cC, \cH_f) = 0$ --- although we believe it is not far from that. To deal with the possible removal of clusters from $\cH_f$, we only consider cluster trees satisfying the following regularity condition.
We say that a cluster tree $\cC$ is closed (for $h = h_f$) if it is closed under intersection and union in the sense that, for any sub-collection of nested clusters $\cS \subseteq \cC$, $\bigcap_{C \in \cS} C \in \cC$ and, if $\inf_{C \in \cS} h(C) > 0$, $\bigcup_{C \in \cS} C \in \cC$.
(Note that this is automatic when $\cS$ is finite, but below we will consider infinite sub-collections.)

\begin{lemma}
\label{lem:closed cluster tree}
Suppose $\cC$ is a closed cluster tree. Then the supremum defining the merge height in \eqref{merge height} is attained, meaning that for any $x, y$ there is $C \in \cC$ containing $x, y$ such that $m_\cC(x,y) = h(C)$. 
\end{lemma}

\begin{proof}
Fix $x,y$ and let $\lambda = m_\cC(x,y)$, assumed to be strictly positive. It suffices to show that there is a cluster that contains these points with height at least $\lambda$.

By the definition in \eqref{merge height}, for any $m \ge 1$ integer, there is $C_m \in \cC$ that contains $x,y$ such that $h(C_m) > \lambda (1-1/m)$. Note that $C_m$ and $C_n$ have at least $x,y$ in common, so that they must be nested. Therefore the sub-collection $\{C_m : m \ge 1\}$ is nested, and by the fact that $\cC$ is closed, $C = \bigcap_{m \ge 1} C_m$ is a cluster in $\cC$. By monotonicity of $h$, $h(C) \ge h(C_m)$ for all $m$, so that $h(C) \ge \lambda$.
\end{proof}

%To simplify the situation even further, we assume that the upper level sets of $f$ are the closure of their interior. This then implies that the same is true of any of its Hartigan clusters. In particular, this precludes $f$ from being constant on any open sets (i.e., its graph is not `flat anywhere'). 

To simplify things further, we just avoid talking about what happens within level sets. 
We will use the following results. 

\begin{lemma}
\label{lem:cc}
For any $s, t > 0$, the connected components of $\{f > s\}$ and those of $\{f\ge t\}$ are either disjoint or nested. 
\end{lemma}

\begin{proof}
Let $R$ be a connected component of $\{f>s\}$ and let $C$ be a connected component of $\{f\ge t\}$. Assume they intersect, i.e., $C \cap R \ne \emptyset$.
First, assume that $s < t$.
In that case $C \subseteq \{f > s\}$, and being connected, there is a unique connected component of $\{f > s\}$ that contains it, which is necessarily $R$.
The reasoning is similar if $s \ge t$. 
Indeed, in that case $R \subseteq \{f \ge t\}$, and being connected, there is a unique connected component of $\{f \ge t\}$ that contains it, which is necessarily~$C$.
\end{proof}

Recall that a mode is simply a local maximum with strictly positive value, i.e., it is a point $x$ such that $f(x) > 0$ and $f(x) \ge f(y)$ whenever $d(x,y) \le r$ for some $r > 0$.

\begin{lemma}
\label{lem:modes}
Consider a continuous function $f$ with bounded upper level sets. Then each connected component of any of its upper level sets contains at least one mode.
\end{lemma}

\begin{proof}
Take $C \in \text{cc}(\up_\lambda)$ for some $\lambda > 0$. 
Because $f$ is continuous and $\up_\lambda$ is compact, $C$ is compact, so that there is $x_0 \in C$ such that $f(x_0) = \max_C f$. Because the distance function is continuous,\footnote{As is well-known, $|d(x,y) - d(x',y')| \le d(x,x') + d(y,y')$ by a simple use of the triangle inequality, so that $d: \Omega \times \Omega \to \bbR$ is Lipschitz and continuous when equipping $\Omega \times \Omega$ with the product topology.} and the fact that all connected components of $\up_\lambda$ are compact, we have that $d(C, C') := \min_{x \in C, x' \in C'} d(x,x') > 0$ for all $C' \in \text{cc}(\up_\lambda)$, so that there is $\eta > 0$ such that $d(C, C') > \eta$ for all $C' \in \text{cc}(\up_\lambda)$. Now, consider $y$ within distance $\eta$ of $x_0$. If $y\in C$, then $f(y) \le \max_C f = f(x_0)$; and if $y \notin C$, then $y \notin \up_\lambda$, and therefore $f(y) < \lambda \le f(x_0)$. We can conclude that $x_0$ is a mode. 
\end{proof}

\begin{lemma}
\label{lem:finitely many modes}
Consider a continuous function $f$ with bounded upper level sets and locally finitely many modes. Then, $f$ satisfies the following property: 
%its upper level sets all have finitely many connected components, which in turn implies that
%for any $\lambda > 0$, if $\eps> 0$ is small enough, $\{f>\lambda+\eps\}$ has the same number of connected components as $\{f>\lambda\}$; in fact, the following stronger property holds:
\begin{equation}
\label{finitely many modes}
\begin{gathered}
\text{For every $\lambda > 0$, if $\eps>0$ is small enough, each connected component of $\{f>\lambda\}$} \\ \text{contains exactly one connected component of $\{f > \lambda+\eps\}$.}
\end{gathered}
\end{equation}
\end{lemma}

%HERE
\begin{proof}
Take any upper level set $U$. Since $U$ is bounded, it can only include a finite number of modes. And since each of its connected components contains at least one mode by \lemref{modes}, it must be the case that $U$ has at most as many connected components as it contains modes.

We now assume that the upper level sets all have finitely many components, and show that \eqref{finitely many modes} holds.
We do so by contradiction. Therefore, assume that \eqref{finitely many modes} does not hold so that there is $\lambda > 0$ and $R$ a connected component of $\{f > \lambda\}$, and a sequence $(\eps_n)$, which we can take to be decreasing and converging to zero, such that for each $n$, $R$ contains at least two connected components of $\{f \ge \lambda+\eps_n\}$. Because $R$ is a bounded region, applying the first part of the statement we find that $R$ can only contain finitely many components of $\{f > \lambda+\eps_n\}$, denoted $A^n_1, \dots, A^n_{m_n}$, with $2 \le m_n \le M$ for all~$n$, where $M$ is the number of modes within $\up_\lambda$. By taking a subsequence if needed, we may further assume that $m_n = m \ge 2$ for all $n$. By the usual nesting property, at every $n$, for each $i$, there is exactly one $j$ such that $A^n_i \subseteq A^{n+1}_j$, and so that we may choose the indexing in such a way that $A^n_i \subseteq A^{n+1}_i$ for all $n$ and all $i$. This allows us to define $A_i = \bigcup_n A^n_i$ for $i = 1, \dots, m$. Since $f$ is continuous, $\{f > \lambda + \eps_n\}$ is open, and therefore so are its connected components (since we assume throughout that $\Omega$ is locally connected), and therefore each $A^n_i$ is open, which then carries over to each $A_i$ being open. The $A_i$ are disjoint because $A^n_i \cap A^{n'}_{i'} = \emptyset$ unless $i = i'$. Therefore, because $R = \bigcup_i A_i$, $R$ must be disconnected --- a contradiction.    
\end{proof}

\begin{proposition}
\label{prp:merge zero}
Let $f$ be a continuous density.
Assume that $\cC$ is a closed cluster tree such that $d_M(\cC, \cH_f) = 0$. Then $\cC$ contains $\cH_f$. 
If, in addition, \eqref{finitely many modes} holds, then, for every $C \in \cC$, $\{f > h(C)\} \cap C$ is  some union of connected components of $\{f > h(C)\}$.
%
%
% that is the closure of its interior,  
%\begin{equation}
%\label{merge zero necessary}
%\begin{gathered}
%\text{there is $V \in \cH_f$ and $R_j, j \in J$ connected components of $V^\circ$ such that,} \\
%\text{with $W = \textstyle\bigcup_{j \in J} \overline{R_j}$, $\{f > h(V)\} \cap W \subseteq C \subseteq W$.}
%\end{gathered}
%\end{equation}
\end{proposition}

\begin{figure}[h]
\centering
\begin{tabular}{ccc}
\includegraphics[width=0.4\textwidth]{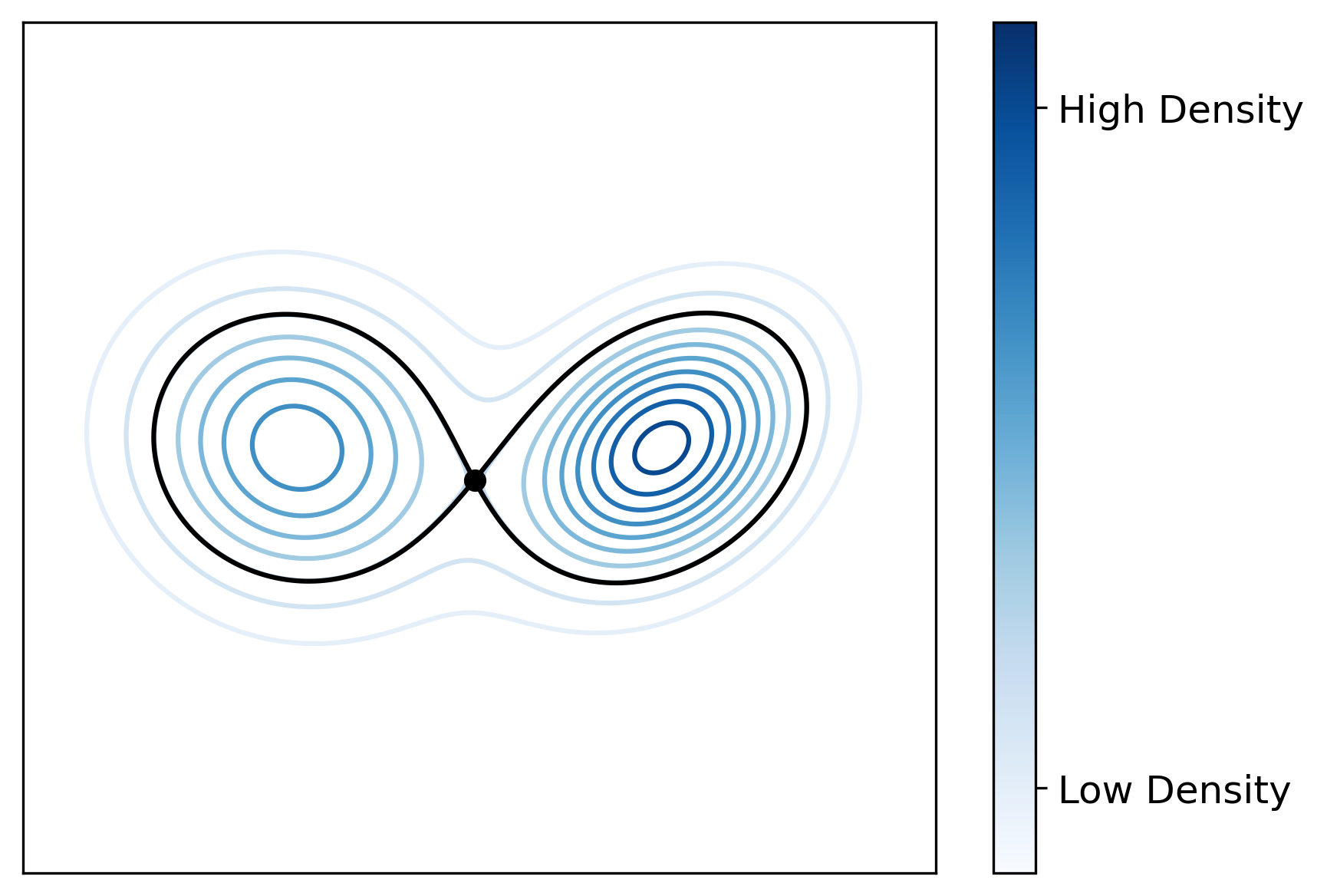} &
\includegraphics[width=0.4\textwidth]{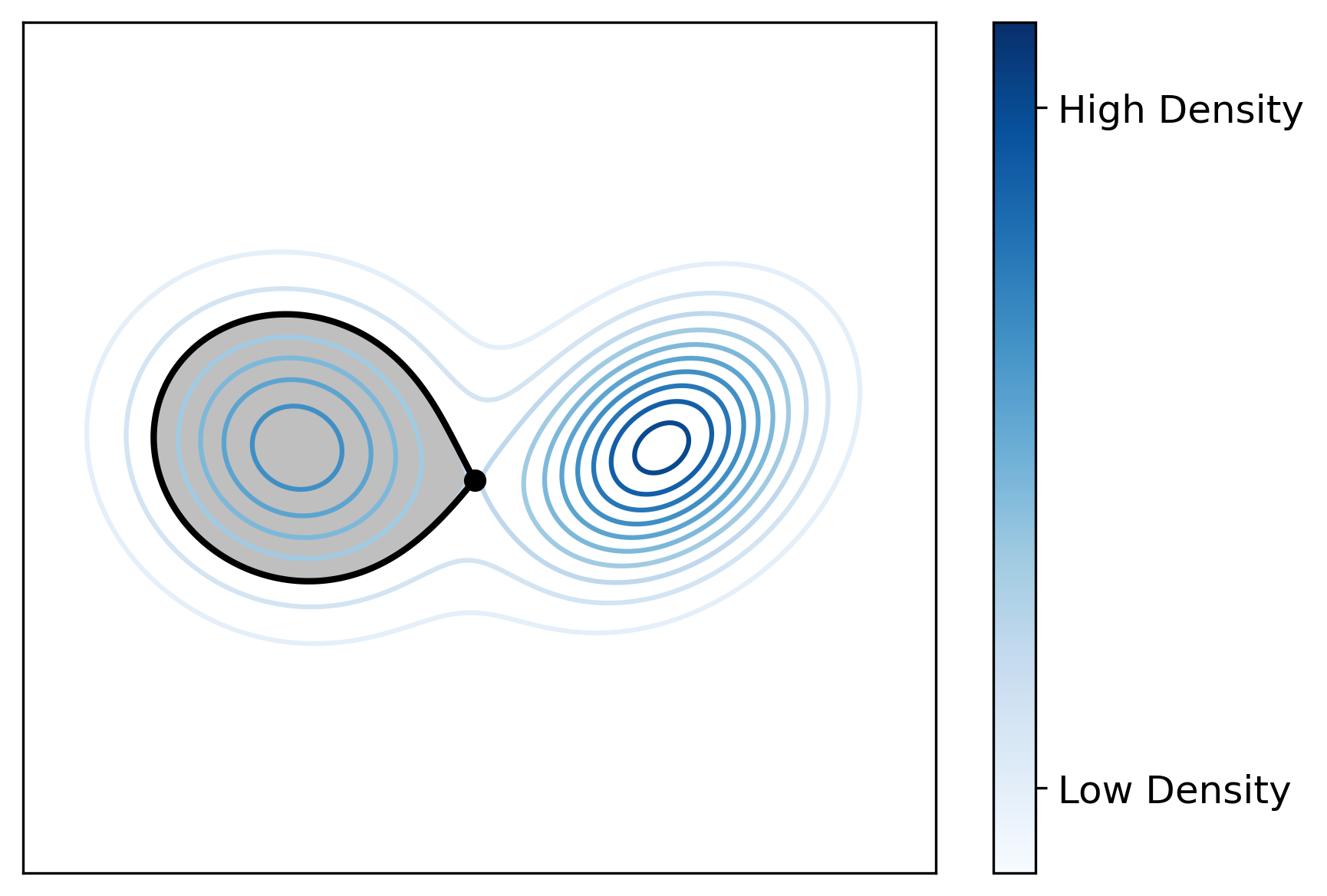}
\end{tabular}
\caption{Left: The level sets of a bimodal Gaussian. The upper level set ``splits" at the level containing the saddle point. (The level set and saddle point are highlighted.) Right: The highlighted cluster is $\overline{R_1}$. The addition of this cluster to $\cH_f$ forms a valid and distinct cluster tree $\cC_1$.}
\label{fig:split_example}
\end{figure}

\begin{proof}
Let $V \in \text{cc}(\up_\lambda)$ for some $\lambda > 0$. Fix $x \in V$ such that $f(x) = \lambda$. Take $y \in V$. First, $m_f(x,y) = \lambda$, and since $m_\cC(x,y) = m_f(x,y)$ and $\cC$ is assumed closed, there is $C_y \in \cC$ containing $x, y$ such that $h(C_y) = \lambda$. 
Note that this implies that $C_y \subseteq V$ since $V$ is the largest connected set that contains $x,y$ such that $h(V) \ge \lambda$.
If $y \ne z$ are both in $V$, we have that $x \in C_y \cap C_z$, so that $C_y$ and $C_z$ are nested. Therefore, the collection $\{C_y : y \in V\}$ is nested, and because $\cC$ is closed, $C = \bigcup_{y \in V} C_y$ belongs to $\cC$. Since  $C_y \subseteq V$ for all $y$, we have $C \subseteq V$; and since $C_y$ contains $y$ for all $y$, we also have $C \supseteq V$; therefore, $C = V$, and we conclude that $V \in \cC$.

For the second part, assume that \eqref{finitely many modes} holds. Take $C \in \cC$ with $\lambda = h(C) > 0$. We want to show that, if $R$ is a connected component of $\{f>\lambda\}$ such that $R \cap C \ne \emptyset$, then $R \subseteq C$.
For $\eps>0$ small enough, $R$ contains exactly one connected component of $\{f > \lambda+\eps\}$, which by way of \lemref{cc} implies that $R$ contains exactly one connected component of $\{f \ge \lambda+\eps\}$, which we denote by $V_\eps$. By the first part of the proposition, which we have already established, $V_\eps$ belongs to $\cC$, and $\cC$ being a cluster tree, we have either $V_\eps \cap C = \emptyset$ or $V_\eps \subseteq C$. Only the latter is possible when $\eps$ is small enough. Indeed, take $x \in R \cap C$, so that $f(x) > \lambda$. Let $\eps>0$ be small enough that $f(x) \ge \lambda+\eps$, so that $x \in V_\eps$. Hence, $V_\eps \subseteq C$ when $\eps>0$ is small enough, and we then use the fact that $R = \bigcup_{\eps>0} V_\eps$ to conclude that $R \subseteq C$.
\end{proof}

%\begin{proof}
%Let $A = R \cap \{f > t\}$ and $B = R \cap \{f \ge t\} = R \cap \up_t$.
%We have that $A \subset \up_t$, and since it is connected by assumption, there $C \in \text{cc}(\up_t)$ such that $A \subseteq C$. We want to show that $B = C$.
%
%We first show that $B$ is connected. 
%\end{proof}

We remark that, when $f$ is `flat nowhere' in the sense that
\begin{equation}
\text{$\overline{\{f > \lambda\}} = \{f \ge \lambda\}$ for any $\lambda > 0$,} 
\end{equation}
then, under the same conditions as in \prpref{merge zero}, any $C \in \cC$ is closure of the union of connected components of $\{f > \lambda\}$. This still leaves the possibility that $\cC \ne \cH_f$, and it can indeed happen --- unless $f$ is unimodal. To see this, for simplicity, suppose that $f$ is `flat nowhere' and has exactly two modes. Assuming that its support has connected interior, there is exactly one level $\lambda > 0$ where the upper level set `splits' in the sense that $\{f > \lambda\}$ has two connected components, say $R_1$ and $R_2$, while $\{f \ge \lambda\}$ is connected. Then, for $j \in \{1,2\}$, $\overline{R_j}$ does not belong to $\cH_f$ and $\cC_j = \cH_f \cup \{\overline{R_j}\}$ is a cluster tree satisfying $d_M(\cC_j, \cH_f) = 0$.  
(Note that $\cH_f \cup \{\overline{R_1}, \overline{R_2}\}$ is not a cluster tree since $\overline{R_1}$ and $\overline{R_2}$ intersect but are not nested.) The situation is illustrated in \figref{split_example}. 

\section{Euclidean Spaces} 
\label{sec:euclidean}

In this section we show that Euclidean spaces have the internally connected partition property by constructing a `shifted' grid that has the required property. 

\begin{figure}[h!]
    \centering
    \includegraphics[width=0.4\linewidth]{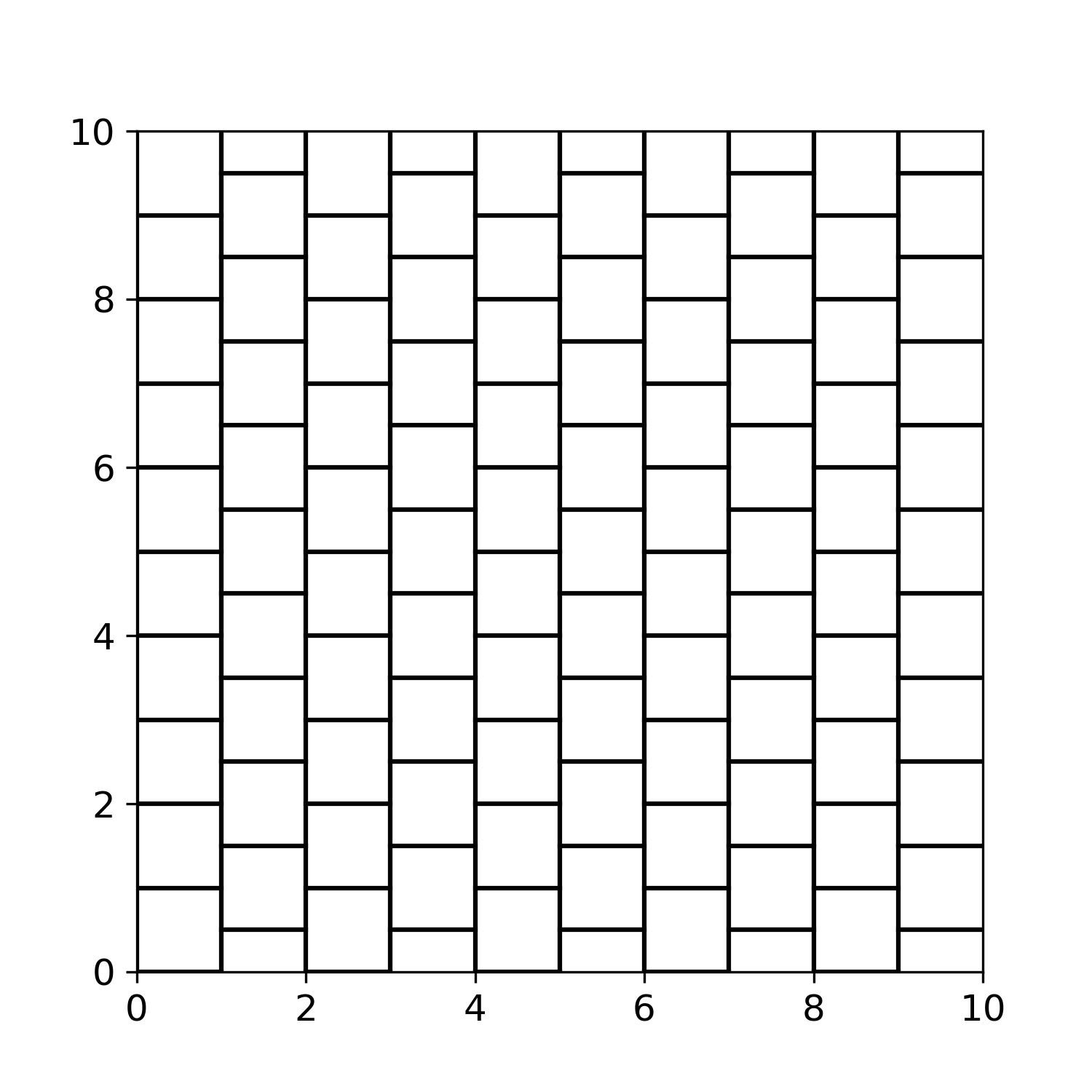}
    \caption{The existence of this shifted grid clearly shows that $\bbR^2$ has the internally connected partition property. This definition of a shifted grid can be extended to higher dimensions to show that $\bbR^d$ has the internally connected partition property for any $d \ge 2$. (This is trivially true in dimension $d=1$ where a regular grid can be used to show that $\bbR$ has the internally connected partition property.)}
    \label{fig:grid}
\end{figure}

\begin{proposition}
\label{prp:euclidean}
Any Euclidean space has the internally connected partition property.
\end{proposition}

\begin{proof}
Consider the Euclidean space $\bbR^d$ (equipped with its Euclidean norm). It is enough to show that there is a a locally finite partition $\{A_i\}$ that has the internally connected property and is such that, for all $i$, $\text{int}(A_i)$ is connected and $\diam(A_i) \le \sqrt{d}$.

Let $\cL_1 = \bbZ$, and for $d \geq 2$, define 
\begin{align}
\cL_{d} = \Big\{ (x_1, \dots, x_d ) : x_d &\in 2\bbZ \text{ and } (x_1, \dots, x_{d-1})  \in \cL_{d-1}; \\\text{ or }  x_d &\in 2\bbZ + 1 \text{ and } (x_1 + \tfrac{1}{2}, \dots, x_{d-1} + \tfrac{1}{2})  \in \cL_{d-1} \Big\}.
\end{align}
For $(x_1, \dots, x_d) \in \cL_d$ define the corresponding cell 
$$A_{(x_1, \dots, x_d) } = [x_1, x_1 +1 ) \times \dots \times [x_d, x_d +1 ).$$
And consider the collection of these cells
\[
\cA_d = \big\{A_{(x_1, \dots, x_d)} : (x_1, \dots, x_d) \in \cL_d\big\}.
\]

Each of these cells has connected interior and has diameter $\sqrt{d}$. 
Moreover, $\cA_d$ is a partition of the entire space $\bbR^d$.
And, as a partition, $\cA_d$ is clearly locally finite. The partition is depicted for $d=2$ in \figref{grid}.

We now prove that $\cA_d$ has the internally connected property. We will proceed by induction on $d$.
For $d=1$, this is clear. 
For $d \ge 2$, assume that $\cA_{d-1}$ has the internally connected property.
Consider $(x_1, \dots, x_d)$ and $(y_1, \dots, y_d)$, both in $\cL_d$, such that $\overline{A_{(x_1, \dots, x_d)}} \cup \overline{A_{(y_1, \dots, y_d)}}$ is connected. We want to show that $\text{int}\big(\overline{A_{(x_1, \dots, x_d)}} \cup \overline{A_{(y_1, \dots, y_d)}}\big)$ is connected too. 
By induction, $\{A_{(z_1, \dots, z_d)} : z_d = x_d\}$ has the internally connected property, so that it is enough to consider a situation where $y_d \neq x_d$. Suppose, for example, that $y_d >  x_d$. 
In that case, the fact that $\overline{A_{(x_1, \dots, x_d)}} \cup \overline{A_{(y_1, \dots, y_d)}}$ is connected implies that $y_d = x_d + 1$ and $y_i = x_i \pm \frac{1}{2}$ for $1 \leq i \leq d-1$.
Further, 
\[\text{int}(\overline{A_{(x_1, \dots, x_d)}} \cup \overline{A_{(y_1, \dots, y_d)}}) = \text{int}(A_{(x_1, \dots, x_d)}) \cup \text{int}(A_{(y_1, \dots, y_d)}) \cup C,\] 
where 
$$C = \Big\{(z_1, z_2, \dots, z_{d-1}, x_d + 1) : x_i + \tfrac{1}{4} + \tfrac{1}{4}\sign(y_i - x_i) \leq z_i \leq  x_i + \tfrac{3}{4} + \tfrac{1}{4}\sign(y_i - x_i)\Big\}.$$
And the fact that $C \subseteq \partial A_{(x_1, \dots, x_d)} \cap \partial A_{(y_1, \dots, y_d)}$ proves that the union above is connected.
\end{proof}

\end{document}